\documentclass{article}
\usepackage{preprint,times}

\usepackage{amsmath,amsfonts,bm}

\def\eqref#1{equation~\ref{#1}}

\def\1{\bm{1}}

\DeclareMathAlphabet{\mathsfit}{\encodingdefault}{\sfdefault}{m}{sl}
\SetMathAlphabet{\mathsfit}{bold}{\encodingdefault}{\sfdefault}{bx}{n}

\usepackage{hyperref}
\usepackage{url}
\usepackage{booktabs}
\usepackage{graphicx}
\usepackage{amsmath,amssymb}
\usepackage{dsfont}
\usepackage{float}
\usepackage{longtable}   %
\usepackage{subcaption}
\usepackage{xcolor}
\usepackage{amsthm}
\usepackage{tikz}
\usetikzlibrary{calc,arrows.meta,positioning,fit,backgrounds,decorations.pathreplacing}

\definecolor{figAccent}{HTML}{0F6377}   %
\definecolor{figAlt}{HTML}{A84E26}      %
\definecolor{figMute}{HTML}{6B7680}     %
\definecolor{figFaceA}{HTML}{E4EFF2}    %
\definecolor{figFaceB}{HTML}{F2EAE4}    %
\definecolor{figRule}{HTML}{3A444C}     %

\tikzset{
  medge/.style     = {draw=figRule, line width=0.7pt, line join=round},
  vtx/.style       = {circle, fill=figRule, inner sep=0pt, minimum size=3.6pt},
  vtxopen/.style   = {circle, draw=figAlt, fill=white, line width=0.8pt,
                      inner sep=0pt, minimum size=4.4pt},
  he/.style        = {draw=figMute, line width=0.75pt, -{Stealth[length=3.6pt]}},
  heMain/.style    = {he, draw=figAccent, line width=1.15pt,
                      -{Stealth[length=4.4pt]}},
  heRel/.style     = {he, draw=figAccent, line width=0.95pt, densely dashed},
  heTwin/.style    = {he, draw=figAlt, line width=1.15pt,
                      -{Stealth[length=4.4pt]}},
  lbl/.style       = {font=\scriptsize, inner sep=1pt},
  lblAccent/.style = {lbl, text=figAccent},
  lblAlt/.style    = {lbl, text=figAlt},
  lblMute/.style   = {lbl, text=figMute},
  chord/.style     = {draw=figAlt, line width=0.9pt, densely dashed},
  panel/.style     = {font=\small\bfseries},
}

\newlength{\heoff}
\newcommand{\halfedge}[3][he]{%
  \path (#2) -- (#3) coordinate[pos=0.20] (@hs) coordinate[pos=0.80] (@he);
  \draw[#1] ($(@hs)!\heoff!90:(#3)$) -- ($(@he)!\heoff!-90:(#2)$);%
}
\newcommand{\helabel}[5][11pt]{%
  \path (#2) -- (#3) coordinate[pos=0.5] (@hm);
  \node[#4] at ($(@hm)!#1!90:(#3)$) {#5};%
}

\tikzset{
  nbox/.style  = {draw=figRule, line width=0.6pt, rounded corners=1.6pt,
                  fill=white, align=center, font=\scriptsize, inner sep=3.4pt},
  nfill/.style = {nbox, fill=figFaceA},
  gbox/.style  = {nbox, fill=figFaceB, draw=figAlt, text=figAlt},
  flow/.style  = {draw=figRule, line width=0.7pt, -{Stealth[length=4.2pt]}},
  gflow/.style = {draw=figAlt, line width=0.9pt, -{Stealth[length=4.2pt]}},
  oper/.style  = {circle, draw=figRule, line width=0.6pt, fill=white,
                  inner sep=0.8pt, font=\scriptsize},
}

\theoremstyle{plain}
\newtheorem{theorem}{Theorem}
\newtheorem{lemma}[theorem]{Lemma}

\title{Playing to Par: Reinforcement Learning for \\
       Provably Optimal Quadrilateral \\
       Block Decompositions}

\author{Arjun Narayanan \\ UC Berkeley \\ arjun.narayanan@berkeley.edu \And
Per-Olof Persson \\ UC Berkeley \\ persson@berkeley.edu}

\newcommand{\NinDist}{96}    %
\newcommand{\Nlarge}{64}     %

\begin{document}
\raggedbottom   %

\maketitle

\begin{abstract}
A quadrilateral block decomposition of a planar domain is judged by whether it
is complete, whether its elements are well shaped, and how many of its vertices
are irregular. The last has a provable floor: the discrete Gauss--Bonnet identity enforces a lower bound on the total vertex irregularity of any all-quadrilateral mesh of a given domain purely based on its topology and corner angles. We train a
reinforcement learning agent to build decompositions that reach this bound,
which we call \emph{par}. It
acts directly on the mesh's half-edge data structure through local edits, with
a policy network whose convolutions follow the mesh's own connectivity, so it
applies unchanged to domains larger than any seen in training. The reward
targets the floor directly, and it is sparse: random play reaches it on no
domain with more than eight sides. We overcome this exploration barrier via behaviour cloning on optimal meshes that are trivial to construct, walked backward into
demonstrations, before training it with PPO. On $\NinDist$ held-out domains the
agent produces an all-quadrilateral mesh on every one, a usable one on $95.7$ on
average, and a provably optimal one on $90$; Gmsh's strongest configuration at
the same element count completes $51$, is usable on $38$ and optimal on none,
and even at three to fourteen times the elements never produces a more regular
mesh. On $\Nlarge$ domains twice the training size the agent completes all, is
usable on $62$, and keeps a median excess over par below one against Gmsh's
$39$ at the same element count.
\end{abstract}

\section{Introduction}
\label{sec:intro}

\begin{figure}[t]
\centering
\includegraphics[width=0.88\textwidth]{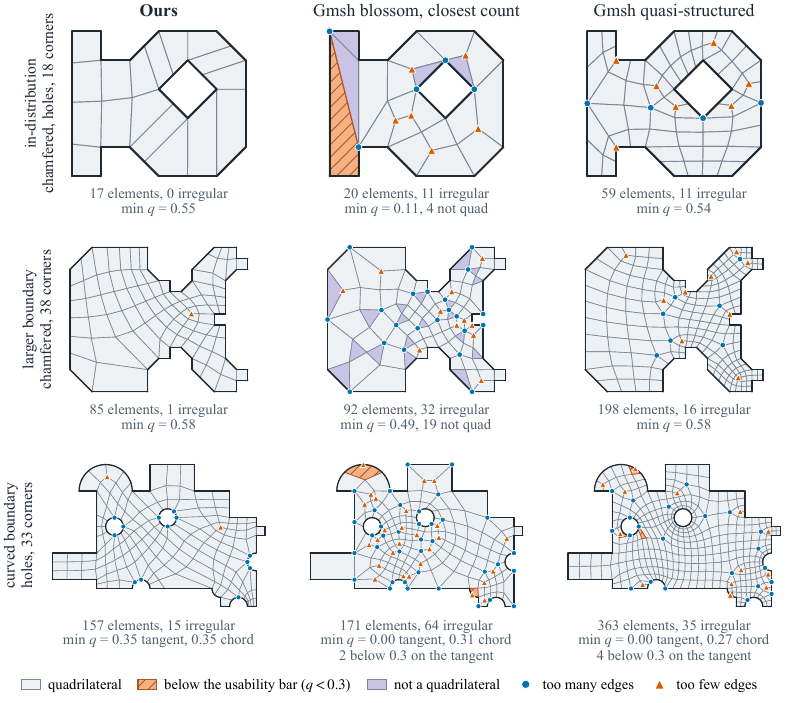}
\caption{%
Our agent and Gmsh on the same three domains. \emph{Top:} an in-distribution
domain (chamfered, with a hole, $18$ corners). \emph{Middle:} a
larger-boundary domain ($38$ corners), beyond any seen in training.
\emph{Bottom:} a curved domain with holes ($33$ corners); the agent was never trained on a
curved boundary, and quality is given against the arc tangents and on chords
(Section~\ref{sec:experiments:curved}). Gmsh
\textsc{blossom} is asked for our element count and returns the closest it can
reach; \textsc{quasi-structured} runs at its natural size. Every irregular vertex is marked. Our top mesh has
none, which the Gauss--Bonnet bound (Section~\ref{sec:problem}) certifies as
optimal; the middle one has one, against $32$ and $16$ for Gmsh, and the curved
one $15$, against $64$ and $35$. Domains are
chosen by the rule of Appendix~\ref{app:gallery}, which shows one per family.}
\label{fig:teaser}
\end{figure}

A block decomposition is a coarse all-quadrilateral partition of a domain with as \emph{regular} a vertex connectivity as the domain allows. Connectivity is crucial: a thin element can be smoothed or split, but a vertex of the wrong valence is a singularity that survives every refinement. Structured and multi-block discretisations, tensor-product bases and subdivision all rely on regular connectivity \citep{bommes2013survey, armstrong2015multiblock}; production meshers clean up valence after the fact \citep{kinney1997cleanup}, and advancing fronts introduce more singularities than a domain needs, though topology makes some unavoidable \citep{armstrong2015multiblock,
fogg2018singularities}. Block decompositions are still built largely by hand or by heuristics, and often do not target optimality.

That optimality criterion exists. The discrete Gauss--Bonnet identity gives every domain a lower bound on the
irregularity of any all-quadrilateral mesh of it, computed from the domain's
corner angles and topology alone (Section~\ref{sec:problem}). We call it \emph{par}. It is
computed once per domain before any mesh exists, and a decomposition that
reaches it is provably optimal in its connectivity. That makes block
decomposition an unusual reinforcement learning (RL) problem: the reward is sparse,
but success carries a certificate.

The sparsity is severe. We pose the problem as an MDP whose state is a
half-edge mesh and whose actions are four local operations
(Section~\ref{sec:mdp}); uniform random play reaches par on about three percent
of five- and six-sided polygons and on none with more than eight sides, and a
policy trained from scratch rarely leaves that regime. We overcome this exploration barrier by behaviour cloning on known optimal trajectories, and continuing with reinforcement learning. The policy is a convolution on the half-edge structure (Section~\ref{sec:policy}): each half-edge receives messages from its \texttt{next}, \texttt{previous} and \texttt{twin}, the pointers that define local topological connectivity. The input features and state representation are designed to allow the agent to scale to larger meshes than seen during training. Unlike prior reinforcement learning for block decomposition \citep{narayanan2024topological}, whose half-edge framework we build on and which edits an existing all-quadrilateral mesh toward ideal vertex degrees, our agent starts from the bare boundary, targets a bound computed from the domain, includes element quality in its reward, and overcomes the sparse reward by cloning. We also evaluate more comprehensively: on domains with holes, on domains twice the training size and on curved boundaries, against the production mesher Gmsh.

\paragraph{Contributions.}
(i) An action space and a policy network that operate directly on the
half-edge data structure of any unstructured polygonal tiling.
(ii) \textsc{geo2d}, a seeded generator and scoring protocol of mechanical-part outlines with chamfers, fillets, and holes, released as a benchmark suite. (iii) A training recipe for the sparse reward: a shaped reward that targets par
directly and is normalised so that returns share one scale across domain sizes; and optimal meshes that are trivial to construct, such as
polyominoes, walked backward into verified demonstrations in the agent's action
space. Cloning on them overcomes the exploration barrier that otherwise blocks
PPO.

\section{Quadrilateral block decomposition and its optimality bound}
\label{sec:problem}

A \emph{domain} $\Omega$ is a bounded planar polygonal region whose boundary
consists of finitely many disjoint simple closed polygons: one outer boundary
and, optionally, $H$ hole boundaries, whose vertices are the \emph{corners} of
$\Omega$. A \emph{quadrilateral block decomposition} of $\Omega$ is a
conforming mesh $\mathcal{M} = (V, E, F)$ covering $\Omega$ in which every face
is bounded by four edges. We write $\deg(v)$ for the number of edges incident
on $v$, $B \subseteq E$ for the boundary edges, and $\chi = |V| - |E| + |F| =
1 - H$ for the Euler characteristic, which depends on $\Omega$ alone.

In a \emph{regular} quadrilateral grid, each face presents an angle of $\pi/2$ at each corner, so the number of faces meeting at $v$ is
$k(v) = \operatorname{round}\bigl( \theta_v / (\pi/2) \bigr)$, where $\theta_v$
is the interior angle at $v$ between the incident boundary edges,
taken as $2\pi$ at an interior vertex. A cycle of $k$ faces about an interior
vertex uses $k$ edges and a fan of $k$ faces at a boundary vertex uses $k+1$,
giving the \emph{desired degree} $d(v)$ and the \emph{irregularity} $I(\mathcal{M})$ of a mesh,
\begin{equation}
  d(v) =
  \begin{cases}
    k(v), & v \in \operatorname{int}\Omega, \\
    \max\{k(v) + 1,\, 2\}, & v \in \partial\Omega,
  \end{cases}
  \qquad\qquad
  I(\mathcal{M}) = \sum_{v \in V} \bigl| \deg(v) - d(v) \bigr| ,
  \label{eq:desired}
\end{equation}
the floor at $2$ recording that a boundary vertex always carries its two
boundary edges. A decomposition with $I(\mathcal{M}) = 0$ is fully structured,
but not every domain admits one. How far from zero it must be is fixed by the
discrete Gauss--Bonnet identity for quadrilateral meshes
\citep{peng2013connectivity, peng2014exploring}. Write $g(v) = 4$ at an
interior vertex and $g(v) = 3$ at a boundary vertex, the degree a vertex would
want were no corner of $\Omega$ sited there.

For any conforming all-quadrilateral mesh of $\Omega$ the identity
$\sum_{v \in V} ( g(v) - \deg(v) ) = 4\chi$ holds (Lemma~\ref{thm:gb},
Appendix~\ref{app:par}): vertex, edge and face counts may differ arbitrarily
between two decompositions of one domain, but this signed total cannot move.
Separating the corner-dependent part of $d$ from the uniform part gives the bound.

\begin{theorem}\label{cor:par}
  Let $c(\Omega) = \sum_v \bigl( d(v) - g(v) \bigr)$, summed over the corners of
  $\Omega$. Every conforming all-quadrilateral mesh $\mathcal{M}$ of $\Omega$
  satisfies
  \begin{equation}
    I(\mathcal{M}) \;\geq\; \mathrm{par}(\Omega)
      \;:=\; \bigl| \, c(\Omega) + 4\chi \, \bigr| .
    \label{eq:par}
  \end{equation}
\end{theorem}

Away from the corners $d(v) = g(v)$, so the bound is the triangle inequality
applied to the identity (proof in Appendix~\ref{app:par}). Both terms of \eqref{eq:par} are
fixed before meshing begins --- $c(\Omega)$ is read off the corner angles,
$\chi$ off the topology --- so $\mathrm{par}(\Omega)$ is computed once per
domain and serves as an exact termination test: a decomposition that reaches
it is provably optimal in its connectivity, although the bound does not
guarantee that such a decomposition exists. When
$\theta_v / (\pi/2)$ is a half-integer the rounding in \eqref{eq:desired} is a
tie between two equally admissible corner treatments; such a corner carries
both, with par minimised over the choice (Appendix~\ref{app:par}).

\section{The decomposition MDP}
\label{sec:mdp}

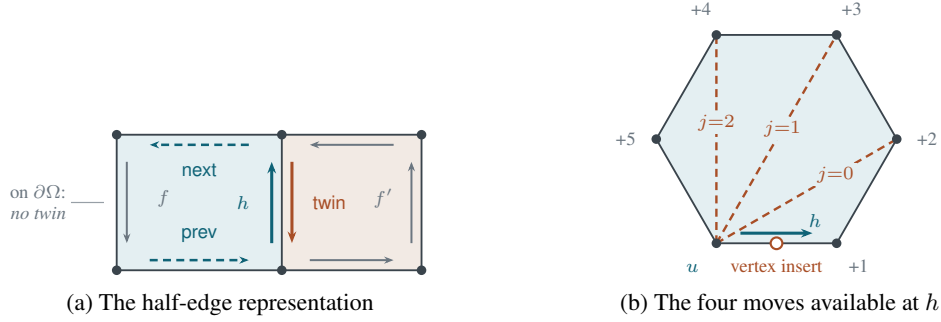
\begin{figure}[t]
\centering
\begin{subfigure}[b]{0.47\textwidth}
\centering
\begin{tikzpicture}[scale=1.12]
  \coordinate (p00) at (0,0);      \coordinate (p10) at (1.95,0);
  \coordinate (p11) at (1.95,1.6); \coordinate (p01) at (0,1.6);
  \coordinate (p20) at (3.6,0);    \coordinate (p21) at (3.6,1.6);

  \fill[figFaceA] (p00) -- (p10) -- (p11) -- (p01) -- cycle;
  \fill[figFaceB] (p10) -- (p20) -- (p21) -- (p11) -- cycle;
  \draw[medge] (p00) -- (p10) -- (p20) -- (p21) -- (p11) -- (p01) -- cycle;
  \draw[medge] (p10) -- (p11);

  \node[lblMute] at (0.55,0.85) {$f$};
  \node[lblMute] at (3.15,0.85) {$f'$};

  \halfedge{p10}{p20}\halfedge{p20}{p21}\halfedge{p21}{p11}\halfedge{p01}{p00}

  \halfedge[heRel]{p11}{p01}
  \halfedge[heRel]{p00}{p10}
  \halfedge[heMain]{p10}{p11}
  \halfedge[heTwin]{p11}{p10}

  \helabel[13pt]{p10}{p11}{lblAccent}{$h$}
  \helabel[12pt]{p11}{p01}{lblAccent}{\textsf{next}}
  \helabel[12pt]{p00}{p10}{lblAccent}{\textsf{prev}}
  \helabel[16pt]{p11}{p10}{lblAlt}{\textsf{twin}}

  \node[lblMute, anchor=east, align=right] at (-0.56,0.80)
       {on $\partial\Omega$:\\\emph{no twin}};
  \draw[figMute!70, line width=0.45pt] (-0.54,0.80) -- (-0.15,0.80);

  \foreach \p in {p00,p10,p11,p01,p20,p21}{\node[vtx] at (\p){};}
\end{tikzpicture}
\caption{The half-edge representation}
\label{fig:dcel}
\end{subfigure}
\hfill
\begin{subfigure}[b]{0.47\textwidth}
\centering
\begin{tikzpicture}[scale=1.12]
  \foreach \i/\a in {0/240, 1/300, 2/0, 3/60, 4/120, 5/180}
    {\coordinate (v\i) at (\a:1.42);}

  \fill[figFaceA] (v0) -- (v1) -- (v2) -- (v3) -- (v4) -- (v5) -- cycle;
  \draw[medge] (v0) -- (v1) -- (v2) -- (v3) -- (v4) -- (v5) -- cycle;

  \draw[chord] (v0) -- (v2);
  \draw[chord] (v0) -- (v3);
  \draw[chord] (v0) -- (v4);
  \node[lblAlt, fill=figFaceA, inner sep=1.2pt] at ($(v0)!0.66!(v2)$) {$j{=}0$};
  \node[lblAlt, fill=figFaceA, inner sep=1.2pt] at ($(v0)!0.55!(v3)$) {$j{=}1$};
  \node[lblAlt, fill=figFaceA, inner sep=1.2pt] at ($(v0)!0.58!(v4)$) {$j{=}2$};

  \halfedge[heMain]{v0}{v1}
  \node[lblAccent] at ($(v0)!0.82!(v1)+(0,0.27)$) {$h$};

  \coordinate (mid) at ($(v0)!0.5!(v1)$);
  \node[vtxopen] at (mid) {};
  \node[lblAlt, anchor=north] at ($(mid)+(0,-0.15)$) {vertex insert};

  \foreach \i in {0,...,5}{\node[vtx] at (v\i){};}
  \node[lblAccent] at (237:1.82) {$u$};
  \foreach \a/\t in {303/{+1}, 0/{+2}, 60/{+3}, 120/{+4}, 180/{+5}}
    {\node[lblMute] at (\a:1.78) {\t};}
\end{tikzpicture}
\caption{The four moves available at $h$}
\label{fig:actions}
\end{subfigure}
\caption{%
\textbf{(a)} A mesh is stored as a doubly connected edge list: every edge carries
two oppositely oriented half-edges, and all connectivity is expressed through
\textsf{next} and \textsf{prev}, which walk a face loop counter-clockwise, and
\textsf{twin}, which crosses to the adjoining face. A half-edge on
$\partial\Omega$ has no twin, and the network is told so by a sentinel distinct
from the one marking a neighbour that merely fell outside the window.
\textbf{(b)} From a half-edge $h$ with source $u$, three actions insert a chord
to the vertex $j+2$ positions ahead in the face loop, cutting off a sub-face of
$j+3$ sides; the fourth splits $h$'s own edge with a new vertex. Chords leaving a
$2$-gon on either side are masked out.}
\label{fig:mdp}
\end{figure}

\subsection{States and observations}
\label{sec:mdp:obs}

We represent a mesh by a doubly connected edge list (DCEL): every edge carries
two oppositely oriented half-edges, and all connectivity is expressed through
three pointers per half-edge --- \texttt{next} and \texttt{previous}, which walk
a face loop counter-clockwise, and \texttt{twin}, which crosses to the adjoining
face (Figure~\ref{fig:dcel}). Holes enter through a slit joining each hole
boundary to the outer one, so a multiply connected domain is still a single
face loop. A state of the MDP is the current mesh together with the half-edge
on which the observation window is centred and the number of moves taken. An
episode begins with $\Omega$ as a single face and ends when the mesh is solved
or the move budget is exhausted.

The mesh grows during an episode, so the observation is restricted to a
fixed-size \emph{template} of $n_T$ half-edges: the centre's face loop, then a
breadth-first traversal of the \texttt{next}/\texttt{previous}/\texttt{twin}
pointers. Three integer arrays give each slot's neighbours in the template,
with two sentinel values distinguishing a neighbour that exists but fell
outside the window from one that does not exist because the half-edge lies on
$\partial\Omega$. Actions are restricted to the template; par, the scores and
the termination test are computed on the whole mesh. Each slot carries ten
features. \emph{Topological:} the desired and actual degree of the source and
of the head vertex of the half-edge, and the number of sides of its face.
\emph{Geometric:} the half-edge's length divided by the template's median, and
the interior angles at its source and head as multiples of $\pi/2$.
\emph{Flags:} whether the source vertex lies on $\partial\Omega$, and whether
it is a corner of $\Omega$. Degrees are clipped (vertices at eight, faces at
fifteen) so that a transient high-valence vertex cannot dominate the input
scale.

A local window cannot say how far the whole mesh is from par, so the
observation also carries a \emph{global vector} of seven scalars: the distance
of the current irregularity from par, the total face defect, the number of
odd-sided faces, par itself, the worst element quality, the mesh size relative
to the window, and the face count, the first five divided by the number of
faces so that each stays in its training range however large the mesh grows.
The fraction of the move budget consumed is a separate scalar. After each move
the centre stays on its face until that face is a quadrilateral, then moves to
the most irregular face in view, or, when nothing irregular is in view, to the
most irregular face or vertex in the whole mesh (Appendix~\ref{app:centre}).

\subsection{Actions}
\label{sec:mdp:actions}

For a half-edge $h$ with source $u$, three actions insert a \emph{chord} from
$u$ to the vertex $j+2$ positions ahead of it in $h$'s face loop,
$j \in \{0,1,2\}$, cutting a sub-face of $j+3$ sides off the face that $h$
bounds; a fourth splits $h$'s edge with a new vertex (Figure~\ref{fig:actions}).
A new vertex takes the generic desired degree $g(v)$ of
Section~\ref{sec:problem}, so par is unchanged by every move. Invalid actions
are masked: no chord may leave a $2$-gon, and no move may close a face of at
most four sides at a boundary vertex that wants three or more edges (a corner of
more than $135^\circ$, including every flat boundary point) when both of its
sides in that face lie on $\partial\Omega$: that element would have an angle of
$180^\circ$ or more there, and is degenerate whatever the smoother does. While our environment also supports deletions,
we disable them to avoid free move-and-undo cycles.
Between moves a few Laplace smoothing sweeps keep the observed angles meaningful; the
edit operations themselves fix only connectivity.

\subsection{Reward and termination}
\label{sec:mdp:reward}

Let $F(s) = \sum_{f} |\,\mathrm{sides}(f) - 4\,|$ be the face defect of a
state, $I(s)$ its irregularity \eqref{eq:desired}, and $q(s)$ the
minimum \emph{shape quality} over the corners of the faces that are already
quadrilaterals, where the shape quality of a corner with edge vectors $a, b$ is
$2\,(a \times b) / (|a|^2 + |b|^2)$: $1$ for a square corner, falling with the
aspect ratio of the two sides, $0$ at a flat corner and negative when the
element folds. The reward is shaped \citep{ng1999potential} by the potential
\begin{equation}
  \Phi(s) \;=\; -\Bigl( F(s) \;+\; w_v \,\bigl|\, I(s) - \mathrm{par}(\Omega) \,\bigr|
              \;+\; w_q \,\max\bigl(0,\; q^{\star} - q(s)\bigr) \Bigr),
  \label{eq:potential}
\end{equation}
with $w_v = 1$, $w_q = 8$ and a quality target $q^\star = 0.4$. The per-move
reward is the increase in potential divided by the initial defect
$D_0 = \max(-\Phi(s_0), 1)$, less a charge per move, plus a \emph{solve bonus} $b$:
\begin{equation}
  r_t \;=\; \frac{\Phi(s_{t+1}) - \Phi(s_t)}{D_0} \;-\; \frac{c}{T_{\max}}
  \;+\; b \,\mathds{1}[\,s_{t+1} \text{ solved}\,], \qquad c = 0.1,\; b = 1 ,
  \label{eq:reward}
\end{equation}
where $\mathds{1}[\cdot]$ is $1$ when its argument holds and $0$ otherwise, so the
solve bonus is paid once, on the move that reaches a solved state (defined
below). Dividing by $D_0$ makes a return the \emph{fraction} of the initial defect
recovered, so large domains do not dominate the advantages, and keeps the
exchange rate between quality and irregularity the same at every size. The
quality term sees only finished quadrilaterals, so each poorly shaped element
is charged at the move that creates it. The move budget scales with the start
state, $T_{\max} = \max(12,\, \lceil 3\,|\mathcal{H}_0| \rceil)$ for $|\mathcal{H}_0|$ half-edges,
and the per-move charge is scaled by it rather than applied as a discount: an
episode that runs to its limit pays $c$ whatever the domain, where a discount
would penalise a large domain for legitimately needing more moves. With
$\gamma = 1$ the return telescopes to
$(\Phi(s_T) - \Phi(s_0))/D_0 - cT/T_{\max} + b\,\mathds{1}[s_T \text{ solved}]$, on
one scale across the distribution, which lets one critic fit every domain.

A state is \emph{solved} when every face is a quadrilateral,
$I(s) = \mathrm{par}(\Omega)$, and the mesh admits a non-degenerate embedding:
it is passed once through an untangling smoother
\citep{escobar2003untangling, freitag1997smoothing} that moves interior
vertices and slides inserted boundary vertices along their edges, keeping the
corners of $\Omega$ fixed (Appendix~\ref{app:geo2d}), and accepted if
$q \geq 0.1$ afterwards. The same test validates certified instances
(Section~\ref{sec:training}), and the same smoother is applied to every method
before scoring (Section~\ref{sec:experiments}). The three quality thresholds
differ deliberately: $q^\star = 0.4$ gives the reward gradient across the range
the agent's meshes occupy, $0.1$ is a degeneracy test, and the usability bar of
$0.3$ used in evaluation is an engineering criterion.

\section{A policy on the half-edge structure}
\label{sec:policy}

\subsection{Convolution over \texttt{next}, \texttt{previous} and \texttt{twin}}

The observation is a set of half-edges together with the three pointers that
relate them, and the network's message passing follows those pointers. A block
reads, for every half-edge in the template, its own feature vector and those of
its \texttt{next}, \texttt{previous} and \texttt{twin}, concatenates the four,
and applies a shared linear map, layer normalisation and a LeakyReLU. A
neighbour that falls outside the window and one that does not exist because the
half-edge lies on $\partial\Omega$ are each replaced by their own learned
vector, so the network can tell the edge of its view from the edge of the
domain. Features are projected to a common width and passed through residual
stages of two blocks each; every block is one hop. No parameter is indexed by
position in the template or by mesh size, so a checkpoint applies unchanged to
a larger window or a larger mesh, and a configuration presents the same input
wherever in the mesh it occurs. An ablation in Appendix~\ref{app:ablations} compares
this network with a Transformer given the same features but not the adjacency.

\subsection{From half-edges to a masked action distribution}

The global vector of Section~\ref{sec:mdp:obs} never enters message passing.
The per-half-edge latents are pooled by mean and by maximum, concatenated with
the global vector and the progress scalar, and mapped by a two-layer MLP to a
\emph{context} vector (Figure~\ref{fig:policy} in Appendix~\ref{app:policyfig}). On the actor side each slot's
latent passes through a shared layer, the context is added to every slot
through a shared linear map, and a shared linear head emits four action logits
per slot; flattened over the template and masked
(Section~\ref{sec:mdp:actions}), these form one categorical distribution over
the action space. Because the context enters after the per-slot nonlinearity
and the head is linear, it shifts every slot's logits by the same vector: it
can re-weight the four action types but cannot favour one half-edge over
another. \emph{Where} to act rests on the convolution's local features and the
centring rule; \emph{what} to do there is what the context can shift. The
critic reads only the context, through its own two-layer MLP, since the value
of a state is a property of the whole mesh and not of the half-edge the window
happens to be centred on; the normalised return of
Section~\ref{sec:mdp:reward} is what lets one such critic fit every domain
size. The network has $96$ channels, four residual stages and about
$3.7 \times 10^5$ parameters.

\section{Training under sparse reward}
\label{sec:training}

Reaching par is a sparse event. Under the MDP of Section~\ref{sec:mdp}, uniform
random play reaches par on $3\%$ of polygons with five or six sides, $2\%$ of
those with up to eight, and on none with more than eight. An explorer that has
never reached par has never seen the only signal that distinguishes optimal
from merely finished, and PPO from a random initialisation reaches it on $7$ of
$\NinDist$ domains after a million steps (ablation in Appendix~\ref{app:ablations}).
We overcome this exploration wall via behaviour cloning on families of meshes which are optimal by construction.

\subsection{Certified instances without search}
\label{sec:training:certified}

A polyomino, a connected set of unit lattice cells, is already an at-par
all-quadrilateral mesh of its outline. Rather than search for a solution we
start from one and undo it: a \emph{backward walk} removes edges and dissolves
vertices, choosing only removals that invert a legal forward move, until a
single face remains, and the reversed record is a move-by-move solution in the
agent's own action space. Four seed families supply variety (polyominoes,
annuli, polar meshes about a hub, and \emph{pinwheel} annuli, whose hole is
rotated half a step against its rim, a shear in polar coordinates), and polyomino
seeds are deformed by an affine map and per-corner jitter within its corners'
angle bins, which leaves its solution valid; moving corners across a bin
boundary yields instances whose bound is above zero
(Appendix~\ref{app:certified}).

\subsection{Cloning, then reinforcement}
\label{sec:training:cloning}

Replaying a certified solution yields (observation, optimal move) pairs.
Because independent moves commute, a state usually has a \emph{set} of equally
optimal moves; we record the set and minimise cross-entropy against a uniform
distribution over it, together with value targets for the critic. The agent is
cloned from $9{,}000$ certified instances (about $187{,}000$ pairs) for six
epochs, reaching $93.7\%$ top-one agreement with the certified sets, with model
selection on a development set of fifteen hand-made levels of simple primitives
(Appendix~\ref{app:devset}).

PPO \citep{schulman2017ppo} then continues from the cloned weights on a mixture
of generated domains (Appendix~\ref{app:geo2d}) and certified ones: $35\%$ of
episodes start from a freshly generated certified instance, so the policy keeps
being asked to solve the problems cloning taught it and keeps meeting instances
whose bound is not zero. Since the cloned value head was fit to returns on a different scale from PPO's,
and an advantage measured against it would be dominated by that offset, we first
train the value head on a frozen policy for forty epochs, and only after that run full PPO. Training runs with
$\gamma = \lambda = 1$ and a target KL of $0.03$ for four million steps
(Appendix~\ref{app:hparams}). The network is small enough that the whole
recipe, generating the certified instances, cloning and four million PPO
steps, takes about four hours on a laptop (Apple M2, eight cores).

\section{Experiments}
\label{sec:experiments}

\subsection{Setup}
\label{sec:experiments:setup}

\paragraph{Domains.}
Training and evaluation outlines come from \textsc{geo2d}, a seeded generator of
lattice-based mechanical parts: rectangular bases with cuts, additions and
combs, optional $45^\circ$ chamfers and interior holes. A defined preset and
a seed determine a domain exactly (Appendix~\ref{app:geo2d}). All evaluation
domains are straight-sided and unseen in training. The \emph{in-distribution}
set is $96$ domains ranging in size from $8$ to $24$ boundary corners drawn from four families: rectilinear or chamfered, with or without holes.
The \emph{larger-boundary} set is $\Nlarge$ domains from the same four families
with $25$ to $50$ corners, beyond anything seen in training. Every generated outline has par zero, but the faces the agent works through on the way do not: each intermediate polygon is a sub-problem with its own bound, non-zero for $39\%$ of them in distribution and $73\%$ on the larger set (Appendix~\ref{app:parzero}).

\paragraph{What is measured.}
Every method is judged on the mesh as delivered, after one pass of the same
untangling smoother. We report, in this order: whether the mesh is
\emph{all-quadrilateral}; whether it is \emph{usable} (shape quality $q \geq 0.3$); the \emph{excess over par}, $I(\mathcal{M}) - \mathrm{par}(\Omega)$ (tables give the median over all-quadrilateral meshes);
and whether it is \emph{at par}, $I(\mathcal{M}) = \mathrm{par}(\Omega)$ with
$q \geq 0.1$, the certificate of Section~\ref{sec:problem}.

\paragraph{Evaluation-time procedure.}
The same procedure is applied to every domain. Each is attempted five times ---
one greedy rollout and four sampled --- with the move budget of
Section~\ref{sec:mdp:reward} doubled to $\lceil 6|\mathcal{H}_0| \rceil$; every
all-quadrilateral state along every attempt is untangled and scored, and the
best is kept, ranked by all-quadrilateral first, then minimum quality, then
closeness to par. If that mesh is below the quality bar, a
\emph{split-and-continue repair} locates its worst corner, makes one of the
environment's own moves there --- a vertex on the longer side of a thin corner,
a chord across a badly angled one --- and hands the state back to the policy
to finish, for up to three rounds. Every count reported is the mean of this
procedure over six rollout seeds on the in-distribution set and four on the
larger set. Across seeds, all-quadrilateral does not move and usable moves by
about one domain; at par moves by $1.3$ of \NinDist{} and $4.1$ of \Nlarge{}
(Appendix~\ref{app:tables}).

\paragraph{Baselines.}
Seven quadrilateral meshing configurations of Gmsh
\citep{geuzaine2009gmsh, gmsh2021quasistructured} are run on the same domains
and scored by the same criteria after the same smoother. Each is run at the
element count the agent used on that domain, so that no method gains by
refining its way to a better score; those that cannot produce a mesh that
coarse are reported at their natural sizes. The elements column is the mean
over a method's all-quadrilateral meshes, and $\times$ours is that mean as a
multiple of ours. We also report a
\emph{head-to-head} win rate: on each domain the agent wins if it produces an
all-quadrilateral mesh and the baseline does not, loses in the reverse case,
and otherwise wins on \emph{quality} if its minimum quality is higher and on
\emph{regularity} if its excess over par is lower (equal values
tie). This avoids comparing medians over the different subsets of domains each
method completes.

\subsection{Main result}
\label{sec:experiments:main}
\label{sec:experiments:gmsh}

The top block of Table~\ref{tab:main} gives the in-distribution result. The agent produces an
all-quadrilateral mesh on every one of the $\NinDist$ domains, a usable one on
$95.7$, and reaches the certified optimum on $90.2$; the median excess over par
is zero on every family. The Gmsh configurations that can be asked
for the agent's element count complete about half the domains, and those that
complete all of them need three to fourteen times the elements to do so, and
still carry a median excess of $7.5$ and $26$. Across all seven configurations
(Table~\ref{tab:gmsh}) the certified optimum is reached three times in $672$
attempts: a quality-driven mesher has no notion of the bound. Head to head the
agent wins on $90\%$ of decided domains against \textsc{blossom} and $90\%$
against \textsc{frontal-quad} at matched count, $67\%$ against
\textsc{quasi-structured} at three times the elements, and $37\%$ against
\textsc{blossom-full} at fourteen times, the one case in which more elements buy
Gmsh the better minimum quality more often than not. On regularity it wins on
every decided domain against all four, at any of their element counts.

\begin{table}[t]
\centering
\footnotesize\setlength{\tabcolsep}{4pt}
\caption{Top: the \NinDist{} in-distribution domains, mean over six evaluation
seeds (per family: Table~\ref{tab:perfamily}). Bottom: the \Nlarge{}
larger-boundary domains, four evaluation seeds; its last rows remove parts of the
evaluation-time procedure (cumulatively) and add attempts. Columns and win rates
as defined in Section~\ref{sec:experiments:setup}; each Gmsh configuration is
asked for our element count and returns the closest it can reach. Bold: best in
column within a block (elements: among methods complete on every domain).}
\label{tab:main}
\label{tab:transfer}
\label{tab:components}
\begin{tabular}{lrrrrrrrr}
\toprule
 & & & & & & & \multicolumn{2}{c}{our win rate} \\
 & all-quad & usable & excess & at par & elements & $\times$ours & quality & regularity \\
\midrule
\multicolumn{9}{l}{\emph{\NinDist{} in-distribution domains ($8$--$24$ corners)}} \\
Gmsh blossom & $51$ & $38$ & $9$ & $0$ & $19$ & $1.2$ & $90\%$ & $100\%$ \\
Gmsh frontal-quad & $43$ & $29$ & $8$ & $1$ & $16$ & $1.0$ & $90\%$ & $100\%$ \\
Gmsh quasi-structured & $\mathbf{96}$ & $\mathbf{96}$ & $7.5$ & $2$ & $51$ & $3.2$ & $67\%$ & $100\%$ \\
Gmsh blossom-full & $\mathbf{96}$ & $\mathbf{96}$ & $26$ & $0$ & $220$ & $14.0$ & $37\%$ & $100\%$ \\
Ours & $\mathbf{96.0}$ & $95.7$ & $\mathbf{0}$ & $\mathbf{90.2}$ & $\mathbf{16}$ & $1.0$ & & \\
\midrule
\multicolumn{9}{l}{\emph{\Nlarge{} larger-boundary domains ($25$--$50$ corners)}} \\
Gmsh blossom & $26$ & $18$ & $39$ & $0$ & $97$ & $1.2$ & $83\%$ & $100\%$ \\
Gmsh frontal-quad & $15$ & $14$ & $34$ & $0$ & $96$ & $1.2$ & $87\%$ & $100\%$ \\
Gmsh quasi-structured & $\mathbf{64}$ & $60$ & $22$ & $0$ & $232$ & $3.0$ & $23\%$ & $100\%$ \\
Gmsh blossom-full & $\mathbf{64}$ & $\mathbf{64}$ & $156$ & $0$ & $1180$ & $15.2$ & $2\%$ & $100\%$ \\
Ours & $\mathbf{64.0}$ & $62.0$ & $\mathbf{0.75}$ & $\mathbf{31.0}$ & $\mathbf{78}$ & $1.0$ & & \\
\quad no repair & $64.0$ & $54.5$ & $0.5$ & $32.8$ & $74$ & & & \\
\quad and default move budget & $58.2$ & $41.0$ & $0$ & $34.8$ & $59$ & & & \\
\quad $17$ attempts & $64.0$ & $63.5$ & $1.75$ & $26.0$ & & & & \\
\bottomrule
\end{tabular}
\end{table}

\subsection{Larger boundaries}
\label{sec:experiments:transfer}
\label{sec:experiments:components}

The bottom block of Table~\ref{tab:main} reports the \Nlarge{} domains with $25$ to $50$
corners, twice anything seen in training, under the same agent and procedure.
The agent completes every domain, clears the quality bar on $62.0$, keeps a
median excess over par below one against Gmsh's $39$ at matched count, and
reaches the certified optimum on $31.0$, where no Gmsh configuration does. Head
to head at matched count it wins on quality on $83\%$ and $87\%$ of decided
domains against \textsc{blossom} and \textsc{frontal-quad}, and on regularity
on all of them, as against every configuration at any element count.

The evaluation-time procedure changes almost nothing in distribution; on the
larger set it matters, and the last rows of Table~\ref{tab:main} remove its
components one at a time. With the default move budget a few domains end at the
move cap with faces still open; doubling it completes all $\Nlarge$ and lifts
the usable count from $41.0$ to $54.5$ at about a quarter more elements, and
tripling buys nothing further. The repair lifts it to $62.0$; one round usually
suffices. Seventeen attempts instead of five raise the usable count to $63.5$. While exact optimality drops on larger domains, with about half at par, the median excess over par stays below one, against $22$ for the most regular Gmsh configuration.

\subsection{Curved boundaries}
\label{sec:experiments:curved}

We run the released agent, unchanged and never trained on a curved boundary, on
$96$ domains of $8$--$24$ corners with fillets and semicircular notches, with and
without holes. Elements along arcs often fall below the quality bar, so we add a test-time repair search. Full details are in Appendix~\ref{app:curved}. Element quality along an arc can be judged against the arc's tangent or the chord connecting two vertices. For block decomposition, the tangent matters, since refinement will place vertices on the curved surface. However, we compute and report both metrics for the sake of comparison since it significantly affects the quality reading of Gmsh.

\begin{table}[t]
\centering
\footnotesize\setlength{\tabcolsep}{3pt}
\caption{Curved boundaries: the released agent without retraining, one evaluation
pass on $96$ domains with $8$--$24$ corners. Left: columns that do not depend on
how quality is measured, as in Table~\ref{tab:main} (at par with $q \geq 0.1$ on
the tangents; on chords \textsc{quasi-structured} reaches par on $2$). Right:
usable count and our win rate on quality, with element quality judged against
the arc tangents or chords. Ours: with the repair
search; the row below it is the five attempts and doubled budget of
Section~\ref{sec:experiments:setup} without any repair. Bold: best in column
(elements: among methods complete on every domain).}
\label{tab:curvedmain}
\begin{tabular}{lrrrrrrrrrr}
\toprule
 & & & & & & & \multicolumn{2}{c}{usable} & \multicolumn{2}{c}{quality win} \\
\cmidrule(lr){8-9} \cmidrule(lr){10-11}
 & all-quad & excess & at par & elem. & $\times$ours & reg.\ win & tangent & chord & tangent & chord \\
\midrule
Gmsh blossom & $62$ & $20$ & $0$ & $60$ & $1.1$ & $100\%$ & $7$ & $54$ & $94\%$ & $54\%$ \\
Gmsh frontal-quad & $55$ & $22$ & $0$ & $58$ & $1.1$ & $100\%$ & $2$ & $42$ & $100\%$ & $78\%$ \\
Gmsh quasi-structured & $\mathbf{96}$ & $10$ & $0$ & $83$ & $1.5$ & $100\%$ & $6$ & $64$ & $95\%$ & $67\%$ \\
Gmsh blossom-full & $93$ & $32$ & $0$ & $235$ & $4.4$ & $100\%$ & $25$ & $79$ & $74\%$ & $47\%$ \\
\midrule
Ours & $\mathbf{96}$ & $\mathbf{2}$ & $\mathbf{37}$ & $\mathbf{54}$ & $1.0$ & & $\mathbf{95}$ & $\mathbf{88}$ & & \\
\quad no repair search & $96$ & $2$ & $19$ & $52$ & & & $74$ & $80$ & & \\
\bottomrule
\end{tabular}
\end{table}

The agent delivers median excess over par of $2$ and reaches
par on $37$ of $96$ domains, while every Gmsh configuration carries an excess of
$10$ to $32$ and reaches par on at most $2$; the agent wins on regularity on every
decided domain against all four (Table~\ref{tab:curvedmain}). It is usable on
$95$ of $96$ against the tangents and $88$ on chords. On chords Gmsh clears the
bar far more often, on up to $79$ with \textsc{blossom-full} at $4.4\times$ the
elements, but at a matched element count the agent remains ahead of
\textsc{blossom} and \textsc{frontal-quad}. On $64$ larger curved domains of
$25$--$50$ corners (Appendix~\ref{app:curved}) the regularity result holds, a
median excess of $8$ against $28$ to $101$, while the agent clears the quality bar on fewer domains than Gmsh configurations that use several times its elements.

\section{Limitations and related work}
\label{sec:related}

\subsection{Limitations}
\label{sec:limitations}

\begin{itemize}\itemsep1pt \parskip0pt \topsep2pt
  \item \emph{The smoother is in the loop.} The solved test untangles the mesh
    before judging it, and a policy trained against that test relies on it; we
    score every method after the same smoother, but an agent that drew a
    well-shaped mesh without one would be preferable.
  \item \emph{Runtime.} Gmsh meshes a domain in milliseconds; our unoptimised
    agent takes under a second per in-distribution domain on a laptop, seconds
    per larger domain, and minutes on the largest holed ones.
  \item \emph{Curved boundaries need a search.} The agent meshes curved
    domains close to the bound without retraining, but clearing the quality bar
    along arcs takes a test-time repair search of up to ten minutes per domain
    (Section~\ref{sec:experiments:curved}); a policy that draws well-shaped
    elements along arcs without it remains open.
\end{itemize}

\subsection{Related work}
\label{sec:related:work}

\paragraph{Learning to mesh.}
Learned quadrilateral meshing has so far learned \emph{where to place the next
element}: \citet{pan2023sac} train a soft actor--critic to make sequential
meshing decisions, \citet{tong2023srlafm} steer an advancing front with
supervised and reinforcement learning, and, concurrently,
\citet{kalyan2026dmsh} pair geometric decomposition with all-quadrilateral
generation in a multi-agent framework. Our actions are instead edits on a
general polygonal half-edge mesh, and an episode ends at a certificate of
optimality rather than when a front is consumed. Closest in objective is
\citet{narayanan2024topological}, whose moves take one all-quadrilateral mesh
to another toward ideal vertex degrees (contrasted in Section~\ref{sec:intro}). Other learned mesh work optimises
geometry or error rather than connectivity \citep{lorsung2023meshdqn,
yang2023marlamr, foucart2023amr, thacher2025delaunay}.

\paragraph{Networks on mesh connectivity; learning backward from solutions.}
Edge and half-edge convolutions \citep{hanocka2019meshcnn,
ludwig2023halfedgecnn} analyse fixed triangle meshes; ours drives a policy on a
polygonal mesh that changes every move, with action heads on the pointers it
convolves over. Learning backward from a solved state is established when there
is one goal \citep{mcaleer2018autodidactic, agostinelli2019deepcubea,
florensa2017reverse, resnick2018backplay}; here each domain has its own
optimum, so certified optima are constructed and walked back to their boundary.

\paragraph{Classical construction and the bound.}
Paving \citep{blacker1991paving}, Q-Morph \citep{owen1999qmorph},
medial-axis decomposition \citep{tam1991medial, fogg2016medial,
sun2021multiblock} and the production meshers descended from them
\citep{geuzaine2009gmsh, gmsh2021quasistructured} decide locally against
quality heuristics; field-aligned methods \citep{bommes2009miq,
kalberer2007quadcover, jakob2015instant, ebke2013qex} optimise alignment to a
cross field. Neither targets a topological bound. The bound itself is due to \citet{peng2013connectivity}
and \citet{peng2014exploring}; we build on their work by making the bound the
termination criterion of an MDP, reached by a learned policy.

\section{Conclusion}
\label{sec:conclusion}

Quadrilateral block decomposition is an unusual reinforcement learning problem:
the space of connectivities grows combinatorially with the boundary and the
reward is sparse, yet every success carries a certificate of optimality.
Exploration almost never reaches that certificate, which stalls PPO from a
random start; cloning on optima that are trivial to construct gets past it, and
PPO on generated domains improves on what cloning taught. The resulting
decompositions are complete, usable and provably as regular as the domain
permits, where production meshers at the same coarseness essentially never
reach the bound, and remain nearly regular at twice the training scale and on
curved boundaries the agent never saw in training.

\subsection*{AI use statement}
We used generative AI tools, primarily large language model coding and writing
assistants, throughout this project, and disclose their use by task.

\emph{Tasks with required disclosure where AI was used.}
\begin{itemize}\itemsep1pt \parskip0pt \topsep2pt
  \item \emph{Implementing methods:} AI was used extensively to write and refactor
    code throughout the codebase.
  \item \emph{Proposing hypotheses:} AI proposed behaviour cloning and helped design the families of constructible meshes with certified solutions.
  \item \emph{Mathematical claims and proofs:} the optimality bound is due to
    prior work \citep{peng2013connectivity, peng2014exploring}; AI assisted with
    writing its proof in the form used here.
  \item \emph{Research methodology and experiments:} the experimental design was
    proposed by the authors. AI ran and monitored the experiments, explored
    hyperparameter settings, including the relative weights of the reward terms,
    and suggested follow-up experiments. The authors analyzed results,
    diagnosed failure modes (often requiring human intuition), and devised mitigations that overcame them.
  \item \emph{Cleaning and reformatting data:} AI was used in result collection and consolidation.
\end{itemize}

\emph{Tasks with required disclosure where AI was not used or which do not apply.}
\begin{itemize}\itemsep1pt \parskip0pt \topsep2pt
  \item Interpreting results: the authors examined the metrics and failure modes
    and drew the conclusions.
  \item Translation, qualitative or thematic data analysis, and surveys or
    interviews: not applicable.
\end{itemize}

\emph{Tasks with recommended disclosure where AI was used.}
\begin{itemize}\itemsep1pt \parskip0pt \topsep2pt
  \item Drafting and revising the text of the paper and figures over several rounds of revision.
  \item Summarising and searching the literature, and identifying relevant work.
\end{itemize}

\emph{What the authors did without AI.}
\begin{itemize}\itemsep1pt \parskip0pt \topsep2pt
  \item The MDP formulation: the environment, the action space, the input
    features, and the state and template construction.
  \item The reward formulation, including its potential-based shaping, and the
    termination condition and acceptance criteria.
  \item The convolution over the half-edge data structure and the action
    distribution.
  \item The experimental design, the interpretation of results, and the design of
    \textsc{geo2d}.
\end{itemize}

\emph{Verification.} We reviewed all AI-assisted work. Every reported number was
checked against the per-domain records it was computed from, bibliographic
entries were checked against the publishers' records, and AI-written code was
tested and its outputs spot-checked by re-running evaluations. We take
responsibility for the final content of this work, including text, claims and
artifacts produced with the aid of generative AI.

\subsection*{Reproducibility statement}
The domain generator is seeded, so every evaluation domain in
Section~\ref{sec:experiments} is reproduced by a preset and a seed; the
certified-instance generator, the environment, the network, the released
checkpoint and its configuration, and the scoring and table scripts are available
at \url{https://github.com/ArjunNarayanan/par-quad}, and
Appendix~\ref{app:hparams} lists every hyperparameter. Scoring is
deterministic given the domain seed and the rollout seed, so any table can be
regenerated exactly; the geo2d domain generator and the per-domain records behind
every table (Appendix~\ref{app:tables}) will be released alongside.

\subsection*{Ethics statement}
This work concerns geometric mesh generation for engineering analysis and
raises no ethical concerns beyond those of general-purpose numerical software.

\bibliography{refs}
\bibliographystyle{preprint}

\appendix
\raggedbottom

\section{The par convention in full}
\label{app:par}

\begin{lemma}[Discrete Gauss--Bonnet \citep{peng2013connectivity, peng2014exploring}]\label{thm:gb}
  For any conforming all-quadrilateral mesh $\mathcal{M}$ of $\Omega$,
  $\;\sum_{v \in V} \bigl( g(v) - \deg(v) \bigr) = 4\chi$.
\end{lemma}

\begin{proof}[Proof of Lemma~\ref{thm:gb}]
  Counting incidences between faces and edges, each of the $|F|$ faces
  contributes four while each interior edge lies on two faces and each boundary
  edge on one, so $4|F| = 2(|E| - |B|) + |B| = 2|E| - |B|$. The boundary is a
  disjoint union of cycles and so carries $|B|$ vertices; summing degrees counts
  each edge twice. Hence $\sum_v \deg(v) = 2|E|$ and
  $\sum_v g(v) = 4(|V| - |B|) + 3|B| = 4|V| - |B|$, and substituting
  $|B| = 2|E| - 4|F|$ into their difference gives
  $4|V| - |B| - 2|E| = 4(|V| - |E| + |F|) = 4\chi$.
\end{proof}

\begin{proof}[Proof of Theorem~\ref{cor:par}]
  A vertex of $\mathcal{M}$ that is not a corner of $\Omega$ is either interior,
  where $d(v) = g(v) = 4$, or lies in the interior of a boundary arc, where
  tangent continuity gives $\theta_v = \pi$ and so $d(v) = g(v) = 3$. Only
  corners contribute, so $\sum_{v \in V} (d(v) - g(v)) = c(\Omega)$ for every
  such mesh, and with Lemma~\ref{thm:gb},
  $\sum_v ( d(v) - \deg(v) ) = 4\chi + c(\Omega)$. The claim is then the
  triangle inequality,
  $\sum_v | d(v) - \deg(v) | \geq | \sum_v ( d(v) - \deg(v) ) |$.
\end{proof}

\paragraph{Ties.}
When $\theta / (\pi/2)$ is exactly a half-integer $k + \tfrac12$ --- a
$135^\circ$ or $225^\circ$ corner under the quadrilateral target, and every
rectilinear corner under a triangular one --- two treatments are equally
admissible: $k$ elements of angle $\theta / k$ or $k+1$ of angle
$\theta / (k+1)$ sit the same distance either side of the target. Rounding
picks one by convention, and the convention then appears in both the
irregularity and the bound. We therefore let such a corner carry the pair
$\{k+1, k+2\}$: its irregularity is the distance to the nearer member, which
keeps $I(\mathcal{M})$ separable over vertices, and par is the smallest value of
$|c + 4\chi|$ over all assignments of one member to each tie corner. Because the
two options differ by exactly one, that minimum is found by a scan over how
many tie corners are taken high rather than an enumeration, and it remains a
bound: the assignment that minimises a given mesh's irregularity has its own
$|c + 4\chi|$ below that irregularity by Theorem~\ref{cor:par}.

\section{The centring rule}
\label{app:centre}

Write the \emph{irregularity} of a face as $|\,\text{sides}-4\,|$ and the
\emph{defect} of a vertex as $|\deg(v) - d(v)|$, to the nearer option at a tie
corner. In the initial state the centre is the half-edge whose source vertex
has the largest defect on the most irregular face. After each move: if the
centre's own face is not yet a quadrilateral, the centre stays where it is and
only the template around it is rebuilt. Otherwise, if any face meeting the
current window is still irregular, the centre moves to one of maximal
irregularity among those and, within it, to the half-edge whose source vertex
has the largest defect. Only when nothing irregular remains in view does the
rule consult the whole mesh, moving to a face of maximal irregularity and its
worst vertex; if every face is quadrilateral, the centre moves to the worst vertex
anywhere.

\section{Certified instances}
\label{app:certified}

A polyomino is an at-par mesh of its outline because every interior vertex has
degree four and every boundary vertex the degree its $90^\circ$, $180^\circ$ or
$270^\circ$ corner asks for. The four seed families contribute different
structure. Polyominoes give rectilinear domains at par zero. Annuli give
domains with a hole; because the backward walk leaves the connecting slit in
place, they arrive in the same single-face representation the agent meets at
test time. Polar meshes of $m$ quadrilaterals about a single hub give par
$|4 - m|$.

Polyomino seeds are then deformed by an affine map plus per-corner jitter,
taken as far as the angle bins allow; seeds with a hole and polar meshes are
only mirrored (and rotated and scaled, which the normalised features do not
see), so their variety comes from the construction itself. Holes in annuli are rotated relative to the outer boundary to create pinwheels. A corner's desired degree depends only on which bin
its interior angle falls in, so the known solution stays valid on the skewed,
irregular polygon. Redrawing an outline from prescribed corner angles is how
instances with par above zero are produced: par is the total rounding error of
the corner angles (Section~\ref{sec:problem}), so moving $k$ corners
consistently across a bin boundary yields an instance of par $k$. About a third
of the instances are of this kind. Every candidate is replayed through the
environment and checked against the solved test of
Section~\ref{sec:mdp:reward} before it enters the dataset.
Figure~\ref{fig:certified} shows one instance of each family.

\begin{figure}[H]
\centering
\includegraphics[width=\textwidth]{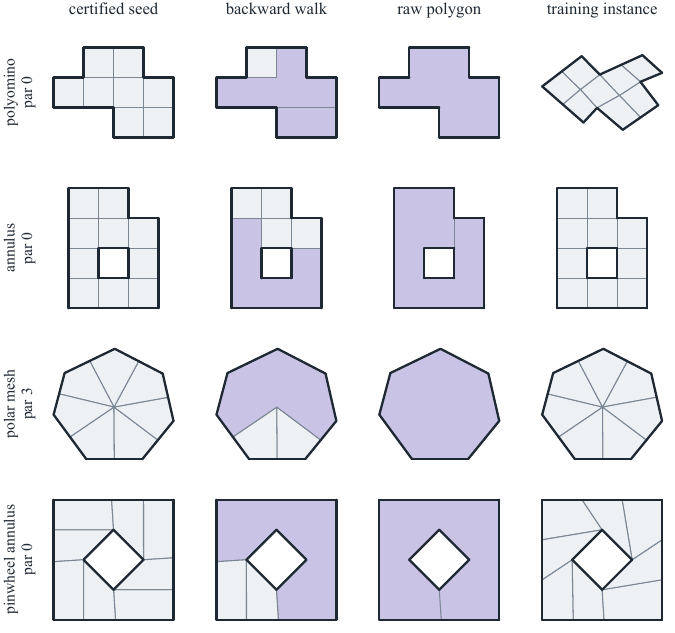}
\caption{Certified instances, one per family, each the first accepted at a
fixed seed. Columns: the constructed optimal mesh; the backward walk a third
of the way through; the single face it ends on, which is the start state; and
the training instance, the transformed outline with the recorded moves replayed
on it and accepted as solved. The last three rows are drawn with rotation and
scale undone. Holed outlines keep one edge joining the hole to the outer
boundary, so that the start state is a single face. Shaded: faces that are not
quadrilaterals.}
\label{fig:certified}
\end{figure}

\section{Par-zero domains are not the easy case}
\label{app:parzero}

Every evaluation domain has par zero (Section~\ref{sec:experiments:setup}),
which might suggest the benchmark never tests the bound on non-zero cases. However, in the course of solving a par zero geometry, the agent routinely encounterss sub-problems with non-zero par. Each chord splits a face in two, and
every face that is not yet a quadrilateral is a domain of its own: a vertex on
its loop with $\deg(v)$ edges, two of them the face's sides, wants
$d(v) - \deg(v) + 2$ edges within it, and Theorem~\ref{cor:par} applied to the
face gives the face's bound. Table~\ref{tab:subproblems} counts the distinct
unfinished faces the released agent creates in its greedy attempt on every
evaluation domain; faces that still contain a hole's slit, whose corners the
count cannot separate, are left out. In distribution, $39\%$ of sub-problems
have a nonzero bound and $75$ of $96$ episodes meet one; on the larger set,
$73\%$ do and every episode meets one.

\begin{table}[H]
\centering
\footnotesize\setlength{\tabcolsep}{4pt}
\caption{Bounds of the sub-problems met in the released agent's greedy attempt
on the par-zero evaluation domains; the last column counts episodes that meet at
least one sub-problem with par above zero.}
\label{tab:subproblems}
\begin{tabular}{lrrrrrrrr}
\toprule
set & sub-problems & par $0$ & $1$ & $2$ & $3$ & $\geq 4$ & par $> 0$ & episodes \\
\midrule
\NinDist{} in-distribution & $2{,}195$ & $60.7\%$ & $33.8\%$ & $4.6\%$ & $0.8\%$ & $0.1\%$ & $39.3\%$ & $75$ of $96$ \\
\Nlarge{} larger & $11{,}717$ & $26.7\%$ & $59.0\%$ & $10.9\%$ & $3.0\%$ & $0.4\%$ & $73.3\%$ & $64$ of $64$ \\
\bottomrule
\end{tabular}
\end{table}

The same holds for whole domains. On $48$ domains with par $1$ to $4$, drawn
from the certified generator at a held-out seed, the agent reaches par on $47$.
No Gmsh configuration at the matched element count reaches par on more than
$8$, and at par $3$ or more only \textsc{packing} reaches it at all, on $4$ of
$20$ (Table~\ref{tab:parset}).

\begin{table}[H]
\centering
\footnotesize
\caption{$48$ held-out domains with par above zero: domains at par, by par.
Ours: best of five attempts, ranked by completion and then closeness to par, at
the default move budget. Gmsh: at the agent's element count where it can reach
it.}
\label{tab:parset}
\begin{tabular}{lrrrrrr}
\toprule
 & par $1$ ($11$) & $2$ ($17$) & $3$ ($17$) & $4$ ($3$) & total & all-quad \\
\midrule
Ours & $\mathbf{11}$ & $\mathbf{16}$ & $\mathbf{17}$ & $\mathbf{3}$ & $\mathbf{47}$ & $\mathbf{48}$ \\
Gmsh quasi-structured & $2$ & $6$ & $0$ & $0$ & $8$ & $\mathbf{48}$ \\
Gmsh packing & $0$ & $3$ & $4$ & $0$ & $7$ & $8$ \\
Gmsh frontal-quad & $1$ & $0$ & $0$ & $0$ & $1$ & $27$ \\
Gmsh blossom & $0$ & $0$ & $0$ & $0$ & $0$ & $30$ \\
Gmsh blossom-full & $0$ & $0$ & $0$ & $0$ & $0$ & $9$ \\
Gmsh simple-recombine & $0$ & $0$ & $0$ & $0$ & $0$ & $2$ \\
Gmsh subdivide & $0$ & $0$ & $0$ & $0$ & $0$ & $\mathbf{48}$ \\
\bottomrule
\end{tabular}
\end{table}

\section{The policy network}
\label{app:policyfig}

\begin{figure}[H]
\centering
\begin{tikzpicture}[scale=1.0]
  \node[nfill, text width=21mm]  (feat) at (0, 0.30) {half-edge features\\$n_T \times 10$};
  \node[nbox,  text width=21mm]  (conn) at (0,-0.85) {\textsf{next}, \textsf{prev}, \textsf{twin}};
  \node[nfill, text width=23mm, minimum height=12mm] (conv) at (2.9,-0.20)
       {DCEL convolution\\\emph{4 residual stages}\\$n_T \times 96$};
  \node[oper] (add) at (7.6,-0.20) {$+$};
  \node[nfill, text width=27mm] (head) at (10.4,-0.20)
       {linear head, $4$ per slot\\$+$ mask $\;\to\; \pi(a \mid s)$};
  \node[lblMute, font=\scriptsize, anchor=north, align=center] at (7.6,-0.52)
       {per-slot linear $+$ LeakyReLU,\\then $+$ context};

  \draw[flow] (feat.east) -- ++(4mm,0) |- (conv.west);
  \draw[flow] (conn.east) -- ++(4mm,0) |- (conv.west);
  \draw[flow] (conv.east) -- (add.west)
        node[midway, above, lblMute, font=\scriptsize] {per-half-edge latents}
        node[midway, below, lblMute, font=\scriptsize] {$n_T \times 96$};
  \draw[flow] (add.east) -- (head.west);

  \node[nfill, text width=18mm] (pool) at (4.7, 1.65) {mean- and\\max-pool};
  \node[oper] (cat) at (6.3, 1.65) {$\mathbin\Vert$};
  \node[nfill, text width=17mm] (ctx) at (7.6, 1.65) {context MLP\\$\to 96$};
  \node[nbox,  text width=27mm] (val) at (10.4, 1.65) {\textbf{critic:} value $V(s)$};

  \draw[flow] (conv.north) |- (pool.west);
  \node[lblMute, font=\scriptsize, anchor=east] at (2.72,1.55) {collapses $n_T$};
  \draw[flow] (pool.east) -- (cat.west);
  \draw[flow] (cat.east)  -- (ctx.west);
  \draw[flow] (ctx.east)  -- (val.west);
  \draw[flow] (ctx.south) -- (add.north)
        node[midway, right, lblMute, font=\scriptsize, xshift=0.5mm, align=left]
        {shared linear map,\\broadcast to every half-edge};

  \node[gbox, text width=21mm] (glob) at (0, 2.85) {global vector (7)\\progress (1)};
  \draw[gflow] (glob.east) -| (cat.north);
  \node[lblAlt, anchor=south, font=\scriptsize\itshape] at (3.5, 2.92)
       {bypasses message passing};
\end{tikzpicture}
\caption{%
The policy. The trunk is per half-edge: convolved features pass through a shared per-slot
linear layer to the action head. Pooling is a \emph{branch} --- it collapses the half-edge axis, and
the global vector and progress scalar, which never enter message passing, join
there. The context vector it produces is one vector per state: the critic's only
input, and, on the actor side, an identical shift added to every half-edge before
a shared linear head. It can therefore re-weight the four action types but cannot
by itself favour one half-edge over another --- selecting \emph{where} to act
rests on the convolution and on the rule that decides which half-edges are in the
window at all.}
\label{fig:policy}
\end{figure}
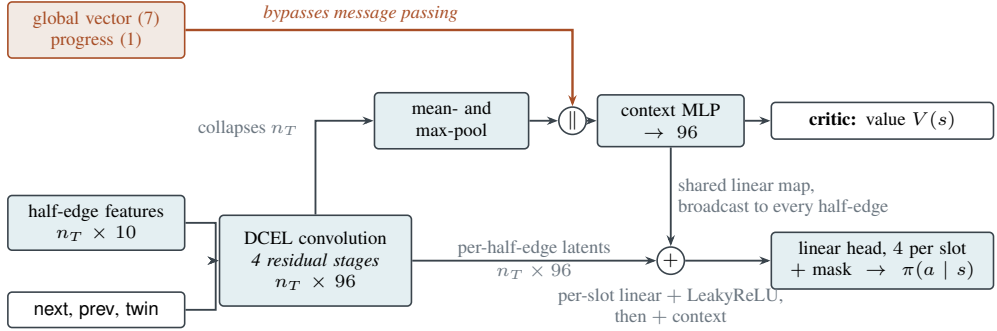

\section{Hyperparameters}
\label{app:hparams}

{\footnotesize
\begin{longtable}{p{0.30\textwidth}p{0.64\textwidth}}
\caption{The released configuration.}\label{tab:hparams}\\
\toprule
\endfirsthead
\multicolumn{2}{l}{\emph{Table~\ref{tab:hparams}, continued}}\\
\toprule
\endhead
\bottomrule
\endlastfoot
observation window $n_T$ & $64$ half-edges \\
edge-length feature & divided by the window's median \\
global scalars & per face, seven entries \\
actions per half-edge & $4$: chords to $+2$, $+3$, $+4$; vertex insertion \\
default move budget $T_{\max}$ & $\max(12, \lceil 3 |\mathcal{H}_0| \rceil)$ \\
smoothing per move & $5$ Laplacian sweeps, corners pinned \\
potential weights $w_v$, $w_q$, $q^\star$ & $1$, $8$, $0.4$ \\
quality metric & shape, $2(a\times b)/(|a|^2+|b|^2)$ \\
step charge $c$, solve bonus $b$ & $0.1$, $1.0$ \\
solved test & topologically at par and $q \geq 0.1$ after $4$ untangling iterations \\
network & $10 \to 96$ projection, $4$ residual stages of $2$ DCEL blocks, LayerNorm, LeakyReLU \\
context & mean- and max-pool ($192$) $+$ global ($7$) $+$ progress ($1$) $\to$ MLP $96$ \\
parameters & $3.7 \times 10^5$ \\
cloning & $9{,}000$ instances, $187$k pairs, $6$ epochs, uniform target over the optimal set \\
certified anchor in PPO & $35\%$ of episodes, fresh batches of $256$ \\
PPO & $\gamma = \lambda = 1$, clip $0.2$, target KL $0.03$, entropy $5\cdot10^{-3}$, value coef.\ $0.5$ \\
 & $512$ steps per worker, $6$--$8$ workers, minibatch $256$, $4$ epochs, lr $10^{-4}$ linear decay \\
critic warm-up & $40$ epochs, policy frozen \\
budget & $4$M PPO steps from the cloned weights \\
training time & about $4$\,h on a laptop (Apple M2, $8$ cores): instances and cloning ${\approx}\,50$\,min, PPO ${\approx}\,3$\,h at ${\approx}\,355$ steps/s \\
\midrule
evaluation: attempts & $5$ ($1$ greedy $+$ $4$ sampled), best kept \\
evaluation: move budget & doubled, $\max(12, \lceil 6 |\mathcal{H}_0| \rceil)$ \\
evaluation: repair & up to $3$ rounds while $q < 0.3$: insert on the long side if aspect $\geq 2$, else chord across the corner; then the policy finishes \\
\end{longtable}}

\section{All seven Gmsh configurations}
\label{app:gmsh}

\begin{table}[H]
\centering
\footnotesize\setlength{\tabcolsep}{4.5pt}
\caption{Seven Gmsh configurations on the \NinDist{} in-distribution domains,
each asked for the agent's element count where it can produce one; the
elements column gives the mean over its all-quadrilateral meshes and its
multiple of our $15.7$. On packing and simple-recombine nearly every win is
decided by completion.
Bold: best in column; for elements, among methods complete on every domain
(packing and simple-recombine average over $8$ and $2$ meshes). Our win rate: fraction of decided domains on which our
mesh is better on quality and on regularity, under the rule of
Section~\ref{sec:experiments:setup}.}
\label{tab:gmsh}
\begin{tabular}{lrrrrrrrr}
\toprule
 & & & & & & & \multicolumn{2}{c}{our win rate} \\
 & all-quad & usable & excess & at par & elements & $\times$ours & quality & regularity \\
\midrule
Gmsh blossom & $51$ & $38$ & $9$ & $0$ & $19$ & $1.2$ & $90\%$ & $100\%$ \\
Gmsh frontal-quad & $43$ & $29$ & $8$ & $1$ & $16$ & $1.0$ & $90\%$ & $100\%$ \\
Gmsh packing & $8$ & $6$ & $5$ & $0$ & $15$ & $0.9$ & $98\%$ & $100\%$ \\
Gmsh simple-recombine & $2$ & $1$ & $11$ & $0$ & $14$ & $0.9$ & $100\%$ & $100\%$ \\
Gmsh subdivide & $\mathbf{96}$ & $43$ & $49$ & $0$ & $85$ & $5.4$ & $97\%$ & $100\%$ \\
Gmsh quasi-structured & $\mathbf{96}$ & $\mathbf{96}$ & $7.5$ & $2$ & $51$ & $3.2$ & $67\%$ & $100\%$ \\
Gmsh blossom-full & $\mathbf{96}$ & $\mathbf{96}$ & $26$ & $0$ & $220$ & $14.0$ & $37\%$ & $100\%$ \\
\midrule
Ours & $\mathbf{96.0}$ & $95.7$ & $\mathbf{0}$ & $\mathbf{90.2}$ & $\mathbf{16}$ & $1.0$ & & \\
\bottomrule
\end{tabular}
\end{table}

Figure~\ref{fig:distributions} shows the same comparison domain by domain, for
the two matched-count configurations and \textsc{quasi-structured}, on both
sets. In the quality panels a mesh that is not all-quadrilateral is counted in
its own bar, so every method is shown on every domain; the per-mesh minimum
quality is the quantity the usability bar is applied to. The excess over par is only
defined on an all-quadrilateral mesh, so its panels show each method's
all-quadrilateral meshes, as a share of them.

\begin{figure}[H]
\centering
\includegraphics[width=\textwidth]{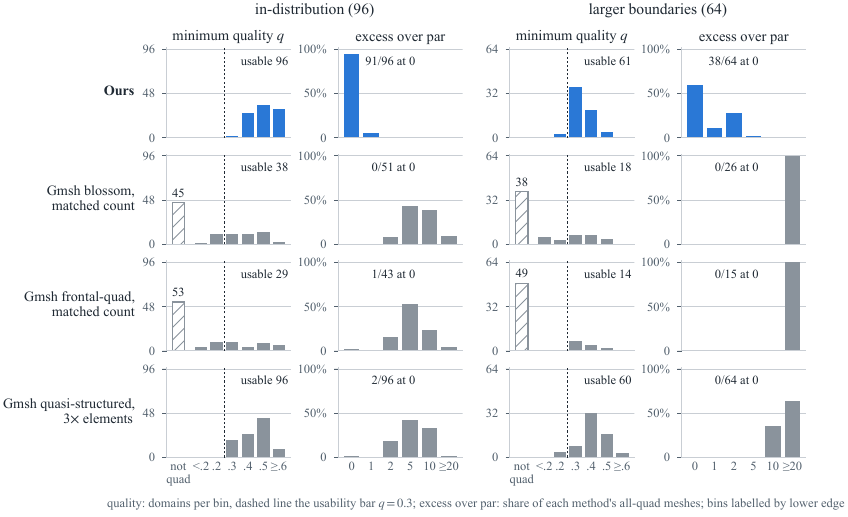}
\caption{Per-mesh minimum quality and excess over par, one mesh
per domain: the agent in one evaluation pass (rollout seed $0$) and Gmsh asked
for that pass's element count on each domain. Quality panels: domains per bin,
hatched bar not all-quadrilateral. Excess panels: share of the method's
all-quadrilateral meshes, with the number at zero quoted. The counts are this
pass's; the tables report the mean over evaluation seeds.}
\label{fig:distributions}
\end{figure}

\begin{table}[H]
\centering
\small
\caption{The in-distribution block of Table~\ref{tab:main} by family: best of five attempts under the procedure,
mean over six evaluation seeds. Columns in the order a mesher is judged:
all-quadrilateral, usable ($q \geq 0.3$), median excess over par,
and at par (the certified optimum). Gmsh \textsc{blossom} is given the agent's
element count on each domain; \textsc{quasi-structured} cannot go that coarse
and runs at its natural size, $3.2\times$ the agent's. Bold: best in column over
all \NinDist{} domains.}
\label{tab:perfamily}
\begin{tabular}{llrrrrr}
\toprule
family & method & all-quad & usable & excess & at par & elements \\
\midrule
chamfered & Ours & $24.0$ & $24.0$ & $0$ & $20.7$ & $14$ \\
 & Gmsh blossom & $14$ & $12$ & $8.5$ & $0$ & $19$ \\
 & Gmsh quasi-structured & $24$ & $24$ & $6.5$ & $0$ & $54$ \\
\addlinespace
rectilinear & Ours & $24.0$ & $24.0$ & $0$ & $23.8$ & $10$ \\
 & Gmsh blossom & $16$ & $10$ & $8$ & $0$ & $15$ \\
 & Gmsh quasi-structured & $24$ & $24$ & $5$ & $1$ & $38$ \\
\addlinespace
chamfered, holes & Ours & $24.0$ & $23.7$ & $0$ & $21.7$ & $18$ \\
 & Gmsh blossom & $14$ & $11$ & $9$ & $0$ & $19$ \\
 & Gmsh quasi-structured & $24$ & $24$ & $9$ & $1$ & $48$ \\
\addlinespace
rectilinear, holes & Ours & $24.0$ & $24.0$ & $0$ & $24.0$ & $21$ \\
 & Gmsh blossom & $7$ & $5$ & $18$ & $0$ & $24$ \\
 & Gmsh quasi-structured & $24$ & $24$ & $12$ & $0$ & $64$ \\
\midrule
all \NinDist & Ours & $\mathbf{96.0}$ & $95.7$ & $\mathbf{0}$ & $\mathbf{90.2}$ & $\mathbf{16}$ \\
 & Gmsh blossom & $51$ & $38$ & $9$ & $0$ & $19$ \\
 & Gmsh quasi-structured & $\mathbf{96}$ & $\mathbf{96}$ & $7.5$ & $2$ & $51$ \\
\bottomrule
\end{tabular}
\end{table}

\section{Per-domain results and evaluation noise}
\label{app:tables}

Every per-domain record behind Tables~\ref{tab:main} and~\ref{tab:ablations} ---
draw index, corner count, par, face defect, excess over par, minimum shape
quality, element and vertex counts, and the Gmsh rows paired by draw --- will be
released with the code as a single archive, together with the two scripts
that regenerate every table and every head-to-head rate from it.
Table~\ref{tab:noise} gives the released agent's spread across evaluation
seeds: six on the in-distribution set and four on the larger one. The domain
seed is fixed, so these deviations measure rollout sampling alone, and a
comparison between two checkpoints on one evaluation seed is paired.

\begin{table}[H]
\centering
\small
\caption{The released agent under the procedure of
Section~\ref{sec:experiments:setup}: mean, standard deviation across
evaluation seeds, and range.}
\label{tab:noise}
\begin{tabular}{llrrr}
\toprule
set & column & mean & sd & range \\
\midrule
\NinDist{} in-distribution, $6$ seeds & all-quad & $96.0$ & $0.0$ & $96$--$96$ \\
 & usable & $95.7$ & $0.5$ & $95$--$96$ \\
 & excess & $0$ & $0.0$ & $0$--$0$ \\
 & at par & $90.2$ & $1.3$ & $88$--$92$ \\
\addlinespace
\Nlarge{} larger, $4$ seeds & all-quad & $64.0$ & $0.0$ & $64$--$64$ \\
 & usable & $62.0$ & $1.0$ & $61$--$63$ \\
 & excess & $0.75$ & $0.43$ & $0$--$1$ \\
 & at par & $31.0$ & $4.1$ & $28$--$38$ \\
\bottomrule
\end{tabular}
\end{table}

\section{The development set}
\label{app:devset}

Model selection during cloning used fifteen hand-made levels --- L, T, U and I
brackets, a staircase, a plus, a triangle and a pentagon, a star, plates with
square and triangular holes, a gear, and three with curved sides --- each with
a known expert solution. They are small (three to ten corners), five of them
are shapes the certified generator can itself produce, and they were never used
for any number in Section~\ref{sec:experiments}. The released agent solves all
fifteen greedily and twelve of the fourteen with a known optimum in the
expert's move count; we report this as a development score only.

\section{Ablations}
\label{app:ablations}
\label{sec:experiments:ablations}
\label{sec:experiments:ablation}

\paragraph{Cloning is necessary for PPO to learn.}
At one million PPO steps and the same seed, the recipe of
Section~\ref{sec:training} is at $91.8$ usable and $76.2$ at par of
$\NinDist$, where the same PPO run from a freshly initialised network completes
$73.5$ but passes the quality bar on $10$ and reaches the bound on $7$, with a
median excess over par of four; on the larger domains nothing it completes
is usable (Table~\ref{tab:ablations}, right). Cloning alone completes $92.8$ and is
usable on $73.5$ but reaches par on only $22.3$; a million PPO steps from it
take that to $76.2$. This is the exploration wall of Section~\ref{sec:training}: without a supply of solved instances the
policy rarely sees the event that the reward targets.

\paragraph{DCEL connectivity has to be an input.}
We train a Transformer encoder with the same parameter count, features,
action space, data, reward and budget, whose input is a breadth-first serialisation of
the template with a positional encoding of slot order. This is a lossy
serialisation of the adjacency. It reaches par on about half as
many domains as the convolution, and already lagged at the end of cloning
(Table~\ref{tab:ablations}, left). We read this as a difference in what the input identifies rather than in capacity, and expect an attention model masked or biased by \texttt{next}, \texttt{previous} and \texttt{twin} to do comparably well.

\paragraph{Run-to-run variance.}
Two training seeds of the released configuration without its quality term, and
two seeds of the whole pipeline (new certified instances, cloning and PPO), differ
in distribution by at most $0.7$ domains on usable and $2.7$ on at par; on the
larger set, by $3$--$5$ on usable and about $10$ on at par. Out of distribution the
seed spread is about five domains, above the evaluation-seed spread of
Table~\ref{tab:noise}, so the larger-boundary numbers should be read as ranges.

\begin{table}[H]
\centering
\footnotesize
\caption{Left: domains at par of the \NinDist{} in-distribution domains, best
of five, at a matched $2$M steps under an earlier configuration of the recipe,
mean of two seeds. Right: PPO from a fresh network, cloning only, and cloning then PPO, at $1$M PPO steps, seed
$1$, best of five at the default move budget and without repair.}
\label{tab:ablations}
\label{tab:ablation-architecture}
\begin{minipage}[t]{0.40\textwidth}
\centering
\begin{tabular}{lrr}
\toprule
extractor & params & at par \\
\midrule
DCEL, $4$ layers & $299$k & $\mathbf{94.5}$ \\
Transf., $4$ layers & $300$k & $51.5$ \\
\bottomrule
\end{tabular}
\end{minipage}\hfill
\begin{minipage}[t]{0.58\textwidth}
\centering
\begin{tabular}{llrrrr}
\toprule
set & training & all-quad & usable & excess & at par \\
\midrule
\NinDist{} & PPO & $73.5$ & $10.2$ & $4$ & $7.2$ \\
 & clone & $92.8$ & $73.5$ & $2$ & $22.3$ \\
 & clone + PPO & $95.5$ & $91.8$ & $0$ & $76.2$ \\
\Nlarge{} & PPO & $13.5$ & $0.0$ & $16.75$ & $0.0$ \\
 & clone & $32.5$ & $0.0$ & $6.6$ & $0.0$ \\
 & clone + PPO & $51.5$ & $14.0$ & $2$ & $16.0$ \\
\bottomrule
\end{tabular}
\end{minipage}
\end{table}

\section{Generating domains, and drawing the mesh}
\label{app:geo2d}

\paragraph{Outlines.}
Training and evaluation domains come from \textsc{geo2d}, a seeded generator of
lattice-based mechanical parts. A shape begins as an occupancy raster on an
integer lattice --- a base rectangle, then rectangular cuts, additions and combs
anchored on the current boundary, each kept only if the material stays connected
and free of pinches and voids --- and tracing that raster gives a
counter-clockwise lattice polygon, to which corner modifiers ($45^\circ$
chamfers for the straight-sided families; fillets and semicircular notches
for the curved ones of Appendix~\ref{app:curved}) and
interior holes with at least one lattice unit of clearance are applied with
exact geometry. One lattice unit is by construction the smallest feature, so a
domain of ratio $R$ fits in an $R \times R$ box and admits a uniform mesh of
unit size, which keeps element counts comparable across a sampled set. A preset
fixes the parameters and a seed determines a domain exactly, so a training
distribution and a held-out evaluation set are specified by a preset and two
seed ranges. The larger-boundary set also caps the ratio of longest to shortest
boundary edge at ten, so a drop there measures size rather than geometric
extremity.

\paragraph{Smoothing.}
The edit operations of Section~\ref{sec:mdp:actions} set connectivity and say
nothing about where nodes sit, so every geometric property of the result is the
smoother's doing. Between moves we run a few Laplacian sweeps, which are cheap
and keep the observed angles meaningful. Where quality decides something --- the
solved test, and every reported number --- the mesh is instead untangled: for
each free node a condition-number objective over its adjacent corners is
minimised by Newton steps, with the corner Jacobian entering through
$h(J) = (J + \sqrt{J^2 + \delta^2})/2$ so that an inverted corner is pushed back
out rather than sending the objective to infinity
\citep{escobar2003untangling}. The corners of $\Omega$ never move; nodes the
agent inserted on a boundary edge slide along it toward the same objective, and
a slide is accepted only if the worst adjacent corner does not get worse
\citep{freitag1997smoothing}. The same smoother, with the same settings, is
applied to every method before it is scored.

\section{Curved boundaries}
\label{app:curved}

\paragraph{The bound on curved domains.}
Section~\ref{sec:problem} states the bound for polygons, but its proof uses no
straightness. A vertex inside an arc has $\theta_v = \pi$ by tangent
continuity, so the only geometry entering \eqref{eq:par} is again the set of
corner angles: a smooth arc contributes to par exactly as a straight edge does.
A square and a disc are both topological discs, but the square's four corners
give $\mathrm{par} = 0$ while the disc, with none, has $\mathrm{par} = 4$.

\paragraph{The repair search.}
Section~\ref{sec:experiments:curved} runs the released agent, trained only on
straight-sided domains, with no change to its weights, observation or action
space. The domains are the four curved suites, with fillets and semicircular
notches: $48$ each at $8$--$24$ and $25$--$50$ corners, with and without holes
($16$ for the holed larger suite), every arc discretised at $45^\circ$. The
agent's five attempts of Section~\ref{sec:experiments:setup}, at the doubled
move budget, supply a pool of all-quadrilateral states, each untangled with the
arcs' tangents. The search then repeats: take the most regular states still
below the quality bar; at the worst element of each apply one edit, either
refining the quadrilateral sheet through it (all-quadrilateral in and out, every
new vertex at its want, so the excess over par is unchanged), or inserting a
vertex on one of its sides or a chord across it and letting the agent finish
the mesh; untangle the results and add them to the pool. It stops after
$600$\,s. The answer is the mesh with the lowest excess over par among those
that clear the bar, ties going to fewer elements, so that refinement is kept
only when it buys the bar. The row without the search in
Table~\ref{tab:curvedmain} selects from the same pool by quality first, as
Section~\ref{sec:experiments:setup} does. The search judges quality with each
corner on an arc measured against the arc's tangent; this is consistent with block decomposition wherein refinement will add new vertices on the arc and not on the chord. However, Gmsh often fails the quality bar when measured on tangents, therefore we report quality measured by the tangent as well as the chord for fairness.

\begin{table}[H]
\centering
\footnotesize\setlength{\tabcolsep}{2.2pt}
\caption{Table~\ref{tab:curvedmain} by suite, with the larger domains; every suite
is curved. Columns, rows and win rates as there; our regularity win rate is
$100\%$ against every configuration in every suite and is omitted. Bold: best in
column within a suite. The holed larger suite has $16$ domains, the others $48$.}
\label{tab:curved}
\begin{tabular}{llrrrrrrrrr}
\toprule
 & & & & & & & \multicolumn{2}{c}{usable} & \multicolumn{2}{c}{quality win} \\
\cmidrule(lr){8-9} \cmidrule(lr){10-11}
suite & method & all-quad & excess & at par & elem. & $\times$ours & tangent & chord & tangent & chord \\
\midrule
$8$--$24$ & Gmsh blossom & $35$ & $18$ & $0$ & $60$ & $1.1$ & $3$ & $31$ & $94\%$ & $50\%$ \\
 & Gmsh frontal-quad & $30$ & $18$ & $0$ & $58$ & $1.1$ & $2$ & $23$ & $100\%$ & $77\%$ \\
 & Gmsh quasi-structured & $\mathbf{48}$ & $9$ & $0$ & $81$ & $1.6$ & $3$ & $33$ & $94\%$ & $71\%$ \\
 & Gmsh blossom-full & $\mathbf{48}$ & $31$ & $0$ & $249$ & $4.8$ & $13$ & $44$ & $73\%$ & $46\%$ \\
 & Ours & $\mathbf{48}$ & $\mathbf{2}$ & $\mathbf{21}$ & $\mathbf{52}$ & $1.0$ & $\mathbf{47}$ & $\mathbf{45}$ &  &  \\
 & \quad no search & $48$ & $2$ & $11$ & $51$ &  & $39$ & $42$ &  &  \\
\addlinespace
holes, $8$--$24$ & Gmsh blossom & $27$ & $24$ & $0$ & $61$ & $1.1$ & $4$ & $23$ & $94\%$ & $58\%$ \\
 & Gmsh frontal-quad & $25$ & $22$ & $0$ & $59$ & $1.1$ & $0$ & $19$ & $100\%$ & $79\%$ \\
 & Gmsh quasi-structured & $\mathbf{48}$ & $14$ & $0$ & $84$ & $1.5$ & $3$ & $31$ & $96\%$ & $62\%$ \\
 & Gmsh blossom-full & $45$ & $32$ & $0$ & $221$ & $4.0$ & $12$ & $35$ & $75\%$ & $48\%$ \\
 & Ours & $\mathbf{48}$ & $\mathbf{2}$ & $\mathbf{16}$ & $\mathbf{55}$ & $1.0$ & $\mathbf{48}$ & $\mathbf{43}$ &  &  \\
 & \quad no search & $48$ & $4$ & $8$ & $54$ &  & $35$ & $38$ &  &  \\
\addlinespace
$25$--$50$ & Gmsh blossom & $28$ & $64$ & $0$ & $185$ & $1.1$ & $2$ & $23$ & $96\%$ & $54\%$ \\
 & Gmsh frontal-quad & $24$ & $52$ & $0$ & $191$ & $1.2$ & $0$ & $19$ & $100\%$ & $60\%$ \\
 & Gmsh quasi-structured & $\mathbf{48}$ & $25$ & $0$ & $323$ & $2.0$ & $0$ & $33$ & $100\%$ & $31\%$ \\
 & Gmsh blossom-full & $\mathbf{48}$ & $98$ & $0$ & $1064$ & $6.5$ & $10$ & $\mathbf{41}$ & $79\%$ & $15\%$ \\
 & Ours & $\mathbf{48}$ & $\mathbf{8}$ & $\mathbf{1}$ & $\mathbf{164}$ & $1.0$ & $\mathbf{36}$ & $30$ &  &  \\
 & \quad no search & $48$ & $8$ & $0$ & $138$ &  & $12$ & $13$ &  &  \\
\addlinespace
holes, $25$--$50$ & Gmsh blossom & $5$ & $60$ & $0$ & $201$ & $1.1$ & $1$ & $5$ & $100\%$ & $81\%$ \\
 & Gmsh frontal-quad & $10$ & $58$ & $0$ & $179$ & $1.0$ & $0$ & $9$ & $100\%$ & $56\%$ \\
 & Gmsh quasi-structured & $\mathbf{16}$ & $32$ & $0$ & $295$ & $1.7$ & $0$ & $10$ & $94\%$ & $25\%$ \\
 & Gmsh blossom-full & $\mathbf{16}$ & $117$ & $0$ & $960$ & $5.4$ & $\mathbf{4}$ & $\mathbf{14}$ & $75\%$ & $19\%$ \\
 & Ours & $\mathbf{16}$ & $\mathbf{13}$ & $\mathbf{1}$ & $\mathbf{178}$ & $1.0$ & $\mathbf{4}$ & $6$ &  &  \\
 & \quad no search & $16$ & $12$ & $1$ & $153$ &  & $2$ & $2$ &  &  \\
\bottomrule
\end{tabular}
\end{table}

On the larger domains the regularity result holds, a median excess over par of
$8$ and $13$ against $25$ to $117$, but the agent clears the quality bar less often on larger domains. One domain per suite is drawn in the gallery (Figure~\ref{fig:gallery-curved}).

\section{Gallery}
\label{app:gallery}

Figures~\ref{fig:gallery-indist} and~\ref{fig:gallery-large} show one domain from
each family of the two evaluation sets. We show domains whose agent result sits closest to the family's medians of corner count, element count and minimum quality. Gmsh \textsc{blossom} is asked for the agent's element count on that domain and returns the closest it can reach; \textsc{quasi-structured} runs at its natural size; both are untangled with the same smoother as the agent's mesh. A marked vertex is one whose degree differs from every degree its position wants (Section~\ref{sec:problem}).

\begin{figure}[H]
\centering
\includegraphics[width=\textwidth]{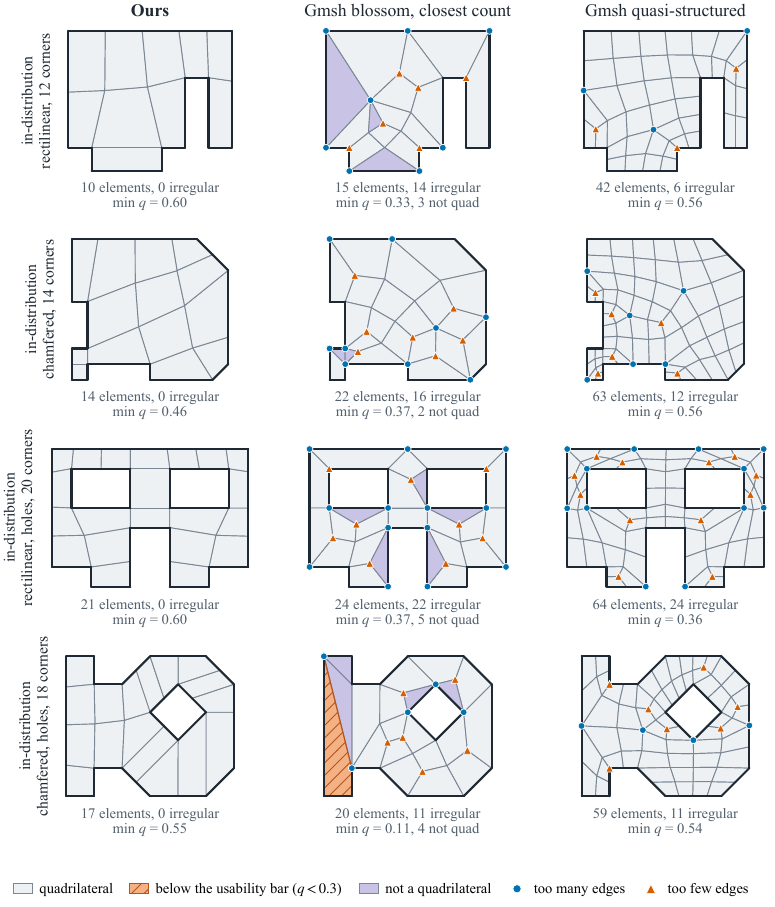}
\caption{One domain per in-distribution family ($8$--$24$ corners).}
\label{fig:gallery-indist}
\end{figure}

\begin{figure}[H]
\centering
\includegraphics[width=\textwidth]{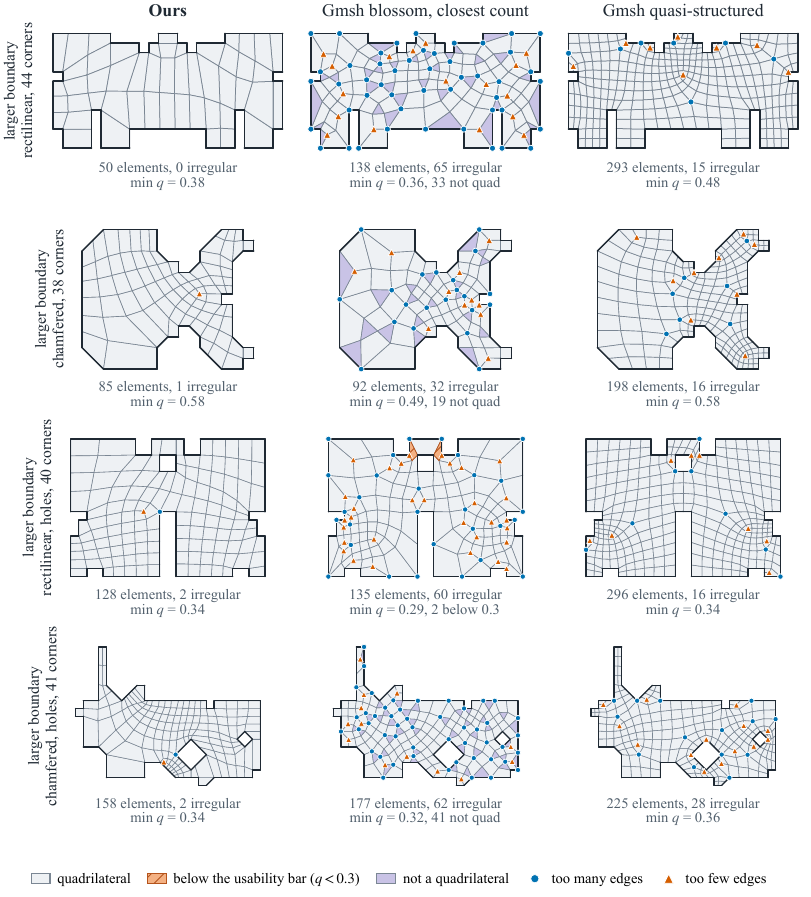}
\caption{Larger-boundary domains ($25$--$50$ corners), beyond any seen in
training: one per family.}
\label{fig:gallery-large}
\end{figure}

Figure~\ref{fig:gallery-curved} shows one domain per curved suite of
Section~\ref{sec:experiments:curved}, meshed by the released agent with the repair
search, chosen by the same rule among domains with at least one arc.

\begin{figure}[H]
\centering
\includegraphics[width=\textwidth]{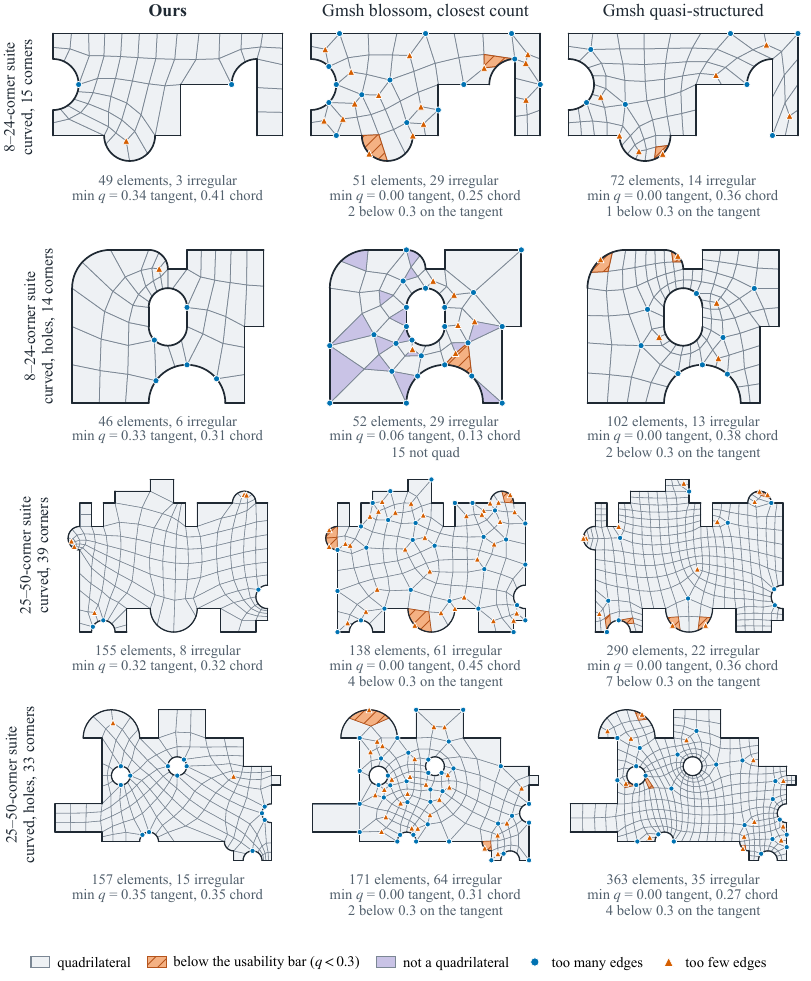}
\caption{Curved domains, one per suite, meshed by the released agent, trained
only on straight-sided domains, with the repair search, and by Gmsh as in
Figure~\ref{fig:teaser}. A boundary edge that lies on an arc is drawn along the
arc; ours has a boundary node every $45^\circ$ of arc. Each panel gives the minimum
element quality judged against the arc tangents and on the straight sides
(Section~\ref{sec:experiments:curved}); the hatching uses the tangent.}
\label{fig:gallery-curved}
\end{figure}

\end{document}